\documentclass[10pt]{article}

\usepackage[T1]{fontenc}
\usepackage[utf8]{inputenc} 
\usepackage{newtxtext}
\usepackage{newtxmath}
\usepackage{microtype}
\usepackage{geometry}
\usepackage{ragged2e}
\justifying

\usepackage{mathtools}

\usepackage{amsthm}

\usepackage{amssymb}

\theoremstyle{plain}
\newtheorem{theorem}{Theorem}[section]
\newtheorem{lemma}[theorem]{Lemma}
\newtheorem{proposition}[theorem]{Proposition}
\newtheorem{corollary}[theorem]{Corollary}
\theoremstyle{definition}
\newtheorem{definition}[theorem]{Definition}

\newtheorem{remark}[theorem]{Remark}

\usepackage{titlesec}
\titleformat{\section}{\normalfont\Large\bfseries}{\thesection}{1em}{}
\titleformat{\subsection}{\normalfont\large\bfseries}{\thesubsection}{1em}{}
\titlespacing*{\section}{0pt}{12pt plus 4pt minus 2pt}{6pt plus 2pt minus 2pt}
\titlespacing*{\subsection}{0pt}{10pt plus 4pt minus 2pt}{4pt plus 2pt minus 2pt}

\usepackage{needspace}
\usepackage{etoolbox}
\pretocmd{\section}{\Needspace{8\baselineskip}}{}{}
\pretocmd{\subsection}{\Needspace{7\baselineskip}}{}{}
\pretocmd{\subsubsection}{\Needspace{6\baselineskip}}{}{}
\pretocmd{\paragraph}{\Needspace{5\baselineskip}}{}{}
\pretocmd{\subparagraph}{\Needspace{4\baselineskip}}{}{}

\AtBeginEnvironment{theorem}{\Needspace{7\baselineskip}}
\AtBeginEnvironment{lemma}{\Needspace{7\baselineskip}}
\AtBeginEnvironment{proposition}{\Needspace{7\baselineskip}}
\AtBeginEnvironment{corollary}{\Needspace{6\baselineskip}}
\AtBeginEnvironment{definition}{\Needspace{6\baselineskip}}
\AtBeginEnvironment{example}{\Needspace{6\baselineskip}}
\AtBeginEnvironment{remark}{\Needspace{5\baselineskip}}

\usepackage{graphicx}
\usepackage{float}
\usepackage[section]{placeins}
\usepackage{flafter}
\usepackage{algorithm}
\usepackage{algpseudocode}
\usepackage{booktabs}
\usepackage{siunitx}
\usepackage{multirow}
\usepackage{subcaption}

\usepackage{authblk}
\usepackage[hidelinks,
  pdfauthor={Tobias Nygaard}
]{hyperref}

\title{pathsig: A GPU-Accelerated Library for Truncated and Projected Path Signatures}

\author{Tobias Nygaard}
\affil{Artificial Intelligence and Mathematics Research Lab, UK \\
\href{mailto:tobias.nygard.24@ucl.ac.uk}{\texttt{tobias.nygard.24@ucl.ac.uk}}}

\date{} 

\usepackage{enumitem}
\usepackage{xparse}
\NewDocumentCommand{\sig}{g g}{%
  \IfNoValueTF{#1}{\mathbf{S}}{%
    \IfNoValueTF{#2}{\mathbf{S}^{#1}}{\mathbf{S}_{#1}^{#2}}%
  }%
}

\newcommand{\R}{\mathbb{R}}
\newcommand{\N}{\mathbb{N}}

\newcommand{\word}[1]{\mathit{#1}}
\newcommand{\eps}{\varepsilon}
\newcommand{\W}{\mathcal{W}}
\newcommand{\loss}{L}

\begin{document}
\maketitle

\begin{abstract}
Path signatures provide a rich representation of sequential data, with strong theoretical guarantees and good performance in a variety of machine-learning tasks. While signatures have progressed from fixed feature extractors to trainable components of machine-learning models, existing libraries often lack the required scalability for large-scale, gradient-based learning. To address this gap, this paper introduces \texttt{pathsig}, a PyTorch-native library that computes path signatures directly in the word basis. By using CUDA kernels to update signature coefficients in parallel over prefix-closed word sets, \texttt{pathsig} achieves high GPU throughput and near-minimal peak memory. Compared with other libraries, \texttt{pathsig} achieves $10$--$30\times$ speedups for computation of truncated signatures and up to $4$--$10\times$ speedups in training that require backpropagation through the signature. Beyond regular truncation, \texttt{pathsig} supports projections of the (infinite-dimensional) signature onto user-specified sets of words and anisotropic truncation motivated by inhomogeneous path regularity, enabling more compact representations that can reduce dimensionality, redundancy, and computational cost.
\end{abstract}


\section{Introduction}
Signature methods have become a widely used tool for working with sequential data in machine learning. 
At their core is the path signature~\cite{chen1958},
$$
\sig_{s,t}(X)
=
1+\sum_{n=1}^{\infty}\int_{s<u_1<\cdots<u_n<t} dX_{u_1}\otimes\cdots\otimes dX_{u_n}
\in T((\mathbb{R}^d)),
$$
which encodes a path through its iterated integrals. The signature yields a rich feature representation of path-valued data, and universal approximation results show that, under suitable conditions, linear functionals of the signature approximate a large class of continuous path-dependent functionals~\cite{signature_approx,cuchiero_uat_cadlag}. The universality of signatures, together with their intrinsic properties, such as invariance under time reparametrisation and the resulting robustness to irregular sampling, makes signatures a principled feature representation for sequential data~\cite{lyons2014roughpaths}.

In practise, one replaces $\sig_{s,t}(X)$ with a finite-dimensional representation, most commonly the truncated signature $\sig^{\le N}_{s,t}(X)\in T_{\le N}(\mathbb{R}^d)$. Compared with many widely used approaches to sequential modelling, computing truncated signatures and differentiating through them can be substantially more expensive. To address this, we introduce \texttt{pathsig}, which combines significant computational improvements with increased flexibility, making path signatures practical and useful across a broader range of applications. We provide \texttt{pathsig} as an open-source, \texttt{pip}-installable PyTorch package for usage with GPUs. Installation and usage are documented on \url{https://pathsig.readthedocs.io}.

\FloatBarrier
\section{Path Signatures, Tensor Algebras, and Word Indexing}
We start with a brief introduction to path signatures, the spaces in which they take values, and the word notation used throughout. For further background, see e.g. \cite{chevyrev2016primer,LyonsMcLeod2022SignatureMethodsML,lyons2007roughpaths, friz2010multidimensional}.

\subsection{Tensor Algebra}\label{subsec:tensor-algebra}
The signature is a graded object whose level-$n$ component lies in the $n$-fold tensor power $(\R^d)^{\otimes n}$. We therefore recall the tensor algebra and its completion, which collect all levels in a single space and whose product preserves this grading.

For $n\ge 1$, let $(\R^d)^{\otimes n}$ denote the $n$-fold tensor product of $\R^d$, and set $(\R^d)^{\otimes 0}:=\R$.

\begin{definition}[Tensor algebra]
The tensor algebra over $\R^d$ is the graded vector space
$$
T(\R^d):=\bigoplus_{n=0}^{\infty}(\R^d)^{\otimes n}.
$$
\end{definition}
Equivalently, $T(\R^d)$ consists of sequences $a=(a_n)_{n\ge0}$ with $a_n\in(\R^d)^{\otimes n}$ and $a_n=0$ for all but finitely many $n$. For $N\in\N$, the truncated tensor algebra is
$$
T_{\leq N}(\R^d):=\bigoplus_{n=0}^N(\R^d)^{\otimes n}.
$$

However, the signature of a path generally does not correspond to a finite sequence of tensors. It belongs to a larger infinite-dimensional space.

\begin{definition}[Completed tensor algebra]
The completed tensor algebra over $\R^d$ is
$$
T((\R^d)):=\prod_{n=0}^{\infty}(\R^d)^{\otimes n}.
$$
\end{definition}
An element $a\in T((\R^d))$ is an infinite sequence $a=(a_0,a_1,a_2,\dots)$ with $a_n\in(\R^d)^{\otimes n}$ for each $n$. We view elements of $T(\R^d)$ as noncommutative polynomials and elements of $T((\R^d))$ as noncommutative formal power series, with tensor level playing the role of degree. Addition and scalar multiplication act levelwise, with multiplication defined in the usual Cauchy sense: for $A=(a_n)_{n\ge0}$ and $B=(b_n)_{n\ge0}$ in $T((\R^d))$,
$$
A\otimes B = (c_n)_{n\ge0},
\qquad
c_n = \sum_{k=0}^{n} a_k \otimes b_{n-k}.
$$

For later reference, let $\pi_{\leq N}$ be the canonical projection $T\left(\left(\mathbb{R}^d\right)\right) \rightarrow T_{\leq N}\left(\mathbb{R}^d\right)$ onto levels $0, \ldots, N$, corresponding to truncation at depth $N$.

\subsection{Path Signatures}\label{subsec:sigs}
We briefly recall the definition of the signature of a path as the collection of its iterated integrals. Throughout, we assume bounded variation so that iterated integrals are understood in the Riemann--Stieltjes sense. More generally, the same definitions and algebraic identities extend to geometric rough paths as limits of smooth (e.g. piecewise linear) approximations; in particular, for continuous semimartingales the resulting iterated integrals are understood in the Stratonovich sense.

\begin{definition}[Path signature]
Let $X:[0,T]\to\mathbb R^d$ be a path of bounded variation and let $\Delta^n[s, t]:=\left\{\left(u_1, \ldots, u_n\right) \in[s, t]^n: u_1<\cdots<u_n\right\}$.
The signature of $X$ over $[s,t]\subseteq[0,T]$ is the formal tensor series
\begin{equation}\label{eq:tensor_sig}
\sig_{s,t}(X) = 1+\sum_{n=1}^{\infty} \int_{\Delta^n[s,t]} dX^{\otimes n}
\in T((\mathbb R^d)), 
\end{equation}
where
$$
\int_{\Delta^n[s,t]} dX^{\otimes n}
= \int_{s<u_1<\cdots<u_n<t} dX_{u_1}\otimes\cdots\otimes dX_{u_n}
\in (\mathbb R^d)^{\otimes n}.
$$
\end{definition}

Accordingly, the truncated signature at depth $N$ is $\sig^{\le N}_{s,t}(X):=\pi_{\leq N} \left(\sig_{s,t}(X)\right)\in T_{\le N}(\R^d)$.

\subsection{Words}
Since the signature takes values in $T\left(\left(\mathbb{R}^d\right)\right)$, which admits a canonical word-indexed basis over $\{1, \ldots, d\}$, we use word and letter notation. This notation simultaneously indexes the tensor basis and specifies the order of integration in the iterated integrals that appear in the signature.

We view the set $\mathcal{A}_d := \{1,\dots,d\}$, which indexes the standard basis $\left\{e_1, \ldots, e_d\right\}$ of $\mathbb{R}^d$, as an alphabet. The elements $i \in \mathcal{A}_d $ are called letters, and finite sequences of letters are called words.

\begin{definition}
A word of length $n \in \mathbb{N}_0$ is a sequence $\word{w} = (i_1, i_2, \ldots, i_n)$ with $i_j \in \mathcal{A}_d$ for each $j$. We use the following notation:
\begin{itemize}[nosep]
    \item The length of a word $\word{w}$ is denoted by $|\word{w}|$.
    \item The empty word, of length 0, is denoted by $\eps$.
    \item $\W_n := \{1, \ldots, d\}^n$ is the set of all words of length exactly $n$.
    \item $\W_{\leq N} := \bigcup_{n=0}^N \W_n$ is the set of all words of length at most $N$.
    \item $\mathcal{W}$ is the set of all words of finite length. 
\end{itemize}
\end{definition}
With this notation, each signature coefficient is an iterated integral indexed by a word $w=(i_1,\dots,i_n)\in\W$, which specifies the (non-commutative) order of integration:
$$
\sig_{s, t}(X, w):=\int_{s<u_1<\cdots<u_n<t} d X_{u_1}^{(i_1)} \cdots d X_{u_n}^{(i_n)},
\qquad \sig_{s,t}(X,\eps)=1.
$$
With the shorthand $e_w := e_{i_1}\otimes\cdots\otimes e_{i_n}$ for $w=(i_1,\dots,i_n)$, the signature can be written as the formal series
$$
\sig_{s,t}(X)=\sum_{w\in\W} \sig_{s,t}(X,w)\, e_w \in T\bigl((\R^d)\bigr).
$$

In addition to indexing signature coefficients, words carry an algebraic structure that encodes how signatures compose. In particular, the tensor-algebra product mirrors concatenation of words, and Chen's relation (Theorem~\ref{thm:chen}) can be expressed in terms of splitting and concatenation of words.

\begin{definition}[Concatenation]
For words $u = (i_1, \ldots, i_m)$ and $v = (j_1, \ldots, j_n)$, their concatenation is
$$
u \circ v = (i_1, \ldots, i_m, j_1, \ldots, j_n) \in \W_{m+n}.
$$
\end{definition}
Following~\cite{Reutenauer1993FreeLieAlgebras}, we say that a word $u\in\W$ is a \emph{factor} of a word $w\in\W$ if there exist words $a,b\in\W$ such that $w = a \circ u \circ b$. If $a=\eps$, then $u$ is called a \emph{prefix} of $w$, and if $b=\eps$, then $u$ is called a \emph{suffix} of $w$. A factor $u$ is \emph{proper} if $u\neq w$.

\section{Signature Computation}
In practice, we work with discretely sampled data $\left\{X_j\right\}_{j=0}^M \subset \mathbb{R}^d$ observed on a partition $0=t_0< \cdots<t_M=T$. An interpolation must therefore be chosen to obtain a continuous path on $[0, T]$. We consider the standard piecewise linear interpolation, for which the signature on each subinterval admits a closed-form expression. Combining these per-interval signatures via Chen’s relation then yields an incremental computation of $\sig_{0,T}(X)$ along the partition.
    
\begin{proposition}\label{prop:sig_over_sub}
Let $0=t_0<\cdots<t_M=T$, and assume that $X:[0,T] \to \mathbb{R}^d$ is affine on each subinterval $\left[t_{j-1}, t_j\right]$. Then for every $j=1, \ldots, M$,
$$
\sig_{t_{j-1},t_j}(X)=\exp\!\big(\Delta X_j\big)\in T((\R^d)),
\qquad \Delta X_j:=X_{j}-X_{j-1},
$$
where $\exp(a):=\sum_{n=0}^\infty \frac{a^{\otimes n}}{n!}$ denotes the tensor exponential.
\end{proposition}

\begin{proof}
Fix $j\in\{1,\dots,M\}$. On $\left[t_{j-1}, t_j\right]$ the path is affine, so for $u \in\left[t_{j-1}, t_j\right]$ we have $d X_u=\left(\Delta X_j /(t_j-t_{j-1})\right) d u$. Thus, for each $n\ge1$,
$$
\sig^{(n)}_{t_{j-1},t_j}(X)
=\Big(\frac{\Delta X_j}{t_j-t_{j-1}}\Big)^{\otimes n}
\int_{t_{j-1}<u_1<\cdots<u_n<t_j} du_1\cdots du_n
=\frac{(\Delta X_j)^{\otimes n}}{n!},
$$
where we used that the simplex $\{t_{j-1}<u_1<\cdots<u_n<t_j\}$ has volume $(t_j-t_{j-1})^n/n!$. With $\sig^{(0)}_{t_{j-1},t_j}(X)=1$, taking the formal sum of $\sig^{(n)}_{t_{j-1},t_j}(X)$ over $n\ge0$ gives the tensor exponential.
\end{proof}

Like the signature, the tensor exponential is an element of $T\left(\left(\mathbb{R}^d\right)\right)$, and we can therefore index its coefficients by words. For $w=(i_1,\dots,i_n)\in \W$,
$$
\exp(\Delta X_j,w):=\frac{1}{n!}\prod_{r=1}^n \Delta X_j^{(i_r)},
\qquad \exp(\Delta X_j,\varepsilon)=1,
$$
and truncation at level $N$ amounts to discarding all terms with $|w|>N$.

With Proposition~\ref{prop:sig_over_sub} providing $\sig_{t_{j-1},t_j}(X)$ in closed form, the signature on $[0, T]$ is obtained by combining the signatures of adjacent subintervals via Chen’s relation.

\begin{theorem}[Chen's relation {\cite[Theorem~7.11]{friz2010multidimensional}}]\label{thm:chen}
Let $X:[0,T]\to\mathbb{R}^d$ be a path of bounded variation. For $0\le s\le u\le t\le T$,
$$
\sig_{s,t}(X)=\sig_{s,u}(X)\otimes \sig_{u,t}(X).
$$
Equivalently, for any word $w\in\W$,
$$
\sig_{s,t}(X,w)=\sum_{w=w_p\circ w_s}\sig_{s,u}(X,w_p)\,\sig_{u,t}(X,w_s),
$$
where the sum is over all decompositions of $w$ into a prefix word $w_p$ and a suffix word $w_s$.
\end{theorem}
This provides a practical method for computing the signature along the partition. In particular, for a piecewise linear path, we have
\begin{equation}
\sig_{0,t_j}(X)=\sig_{0,t_{j-1}}(X)\otimes \exp(\Delta X_j),
\qquad j=1,\ldots,M,
\end{equation}
with the initial condition $\sig_{0,0}(X)=1$. Equivalently, for each word $w\in\W$,
\begin{equation}\label{eq:rec_chen}
\sig_{0,t_j}(X,w)=\sum_{w=w_p\circ w_s}\sig_{0,t_{j-1}}(X,w_p)\,\exp(\Delta X_j,w_s),
\qquad j=1,\ldots,M.
\end{equation}

This iterative formula reduces the computation of the signature to a sequence of tensor multiplications and computations of tensor exponentials. Established libraries like \texttt{iisignature} \cite{iisig}, \texttt{esig} \cite{esig2023}, and \texttt{Signatory} \cite{kidger2020signatory} implement this recursion in C++ backends, offloading tensor operations to HPC (High-Performance Computing) frameworks such as BLAS, Eigen, or PyTorch's ATen for CPU/GPU acceleration. A notable recent addition is \texttt{keras\_sig} \cite{genet2025keras}, which extends multi-backend support (PyTorch, TensorFlow, JAX) in Keras 3, reframing signature computations as parallel matrix multiplications and cumulative products for GPU acceleration, reportedly achieving up to 10-20x speedups in forward passes and up to 55$\%$ faster training times. A further high-performance library is \texttt{pySigLib} \cite{shmelev2025pysiglib}, which implements both the direct recursion of \texttt{iisignature} and the Horner-type scheme used in \texttt{Signatory}, with additional memory-layout and hardware-level optimisations.

\subsection{Computation over Prefix-Closed Sets of Words}
Existing methods for computing the signature organise around operations in the tensor algebra on the graded sequence of tensors
$$\bigl(1, \sig^{(1)}_{0,T}(X), \sig^{(2)}_{0,T}(X), \ldots\bigr).$$
Although such tensor-algebra operations can be accelerated by highly optimised linear-algebra backends, they introduce an additional layer of abstraction, both computationally and mathematically.

Instead, we work directly in the canonical word basis of $T\left(\left(\mathbb{R}^d\right)\right)$. By exploiting the dependency structure in Chen's relation, we decompose the word basis into computational units within which the corresponding signature coefficients can be updated independently. In particular, the dependence on signature terms in \eqref{eq:rec_chen} is solely on the word $w$ itself and all of its proper prefixes. The tensor exponential terms, by contrast, have no time dependence. This motivates decomposing the word basis into sets of words that are prefix-closed.

\begin{definition}[Prefix-closed]
We say that a set of words $W \subseteq \W$ is prefix-closed if for every word $w \in W$, all prefixes of $w$ are also in $W$. That is,
$$
w \in W \implies w_{[k]} \in W \quad \forall k \in \{0, 1, \ldots, |w|\},
$$
where $w_{[k]}:=\left(i_1, \ldots, i_k\right)$ for $k>0$ and $ w_{[0]}:=\eps$.
\end{definition}

Over a prefix-closed set of words $W \subseteq \mathcal{W}$, the coefficients $\left\{\sig_{0, T}(X, w)\right\}_{w \in W}$ are computed by iterating \eqref{eq:rec_chen}. For $w=\left(i_1, \ldots, i_n\right) \in W$, at each step $j=1, \ldots, M$ we evaluate the prefix--suffix sum using Horner's method by computing
$$
h
=\Delta X_j^{(i_n)}\Bigg(
    \sig_{0,t_{j-1}}(X,w_{[n-1]})
    +\frac{\Delta X_j^{(i_{n-1})}}{2}\Big(
        \sig_{0,t_{j-1}}(X,w_{[n-2]})
        +\cdots
        +\frac{\Delta X_j^{(i_1)}}{n}\,\sig_{0,t_{j-1}}(X,w_{[0]})
    \Big)
\Bigg),
$$
and update $\sig_{0,t_j}(X,w)=\sig_{0,t_{j-1}}(X,w)+h$. This is equivalent to \eqref{eq:rec_chen} but avoids explicitly forming the coefficients of $\exp(\Delta X_j)$. In practice, this reduces the number of intermediate floating-point operations and tends to be less sensitive to rounding error at higher truncation levels. Algorithm~\ref{alg:chen_update} summarises the resulting procedure.

\begin{algorithm}[H]
\caption{Signature update via Chen's relation and Horner's method for $w=(i_1,\dots,i_n)$}
\label{alg:chen_update}
\begin{algorithmic}[1]
\Require $\sig_{0,t_{j-1}}(X,w_{[k]})$ for $k=0, 1,\ldots,n$; \, $\Delta X_j^{(i)}$ for all $i \in \{i_1, \dots, i_n\}$.
\Ensure $\sig_{0,t_j}(X,w)$.
\State $h \gets 0$ \Comment{accumulator}
\For{$k = 0,1,\ldots,n-1$} \Comment{build from inner-most to outer-most}
    \State $h \gets \dfrac{\Delta X_j^{(i_{k+1})}}{\,n-k\,}\Big(\sig_{0,t_{j-1}}(X,w_{[k]}) + h\Big)$.
\EndFor
\State $\sig_{0,t_j}(X,w) \gets \sig_{0,t_{j-1}}(X,w) + h$.
\end{algorithmic}
\end{algorithm}
\FloatBarrier

\subsection{Choice of Prefix-Closed Sets}
To compute the truncated signature, we must cover $\mathcal{W}_{\leq N}$ by a collection of prefix-closed sets $\left(W_i\right)_{i=1}^I$, and assign each $W_i$ to a computational unit. In a CUDA implementation this corresponds to choosing a granularity; for instance, one $W_i$ per thread, warp, or block\footnote{In CUDA, a thread is the basic execution unit. Threads are organized into warps (32 threads executing in lockstep), and a thread block consists of one or more warps, providing a set of threads that can synchronize and share on-chip shared memory.}, and having each unit apply Algorithm~\ref{alg:chen_update} to all $w\in W_i$ at each update $t_{j-1}\to t_j$.

Among the assignments we considered, the highest performance was obtained by assigning each thread the prefix-closed set generated by a single word $w$.

\begin{definition}
Let $w = (i_1, \dots, i_n) \in \W_n$. The set of all prefixes of $w$ is
$$
P_w:=\{u \in \W \mid \exists v \in \W,\ w=u \circ v\}.
$$
\end{definition}
Although this choice incurs $O(n^2)$ work per thread at each step and introduces redundancy, its improved memory locality and uniform control flow led it to outperform coarser assignments at the warp or block level, which required more complex work partitioning, launch logic, and mapping between words. Moreover, it allows for computing the signature coefficients of any non-empty word set $\mathcal{I} \subset \mathcal{W}$, enabling projections beyond truncation.
 
\FloatBarrier

\subsection{Log Signatures}\label{subsec:logsigs3}
Alongside signature computation, \texttt{pathsig} supports log-signature computation in the standard Lyndon basis. Recall that the (truncated) signature is \emph{group-like} \cite{friz2010multidimensional}. The set of group-like elements forms a step-$N$ nilpotent Lie group $G_{\le N}$ whose Lie algebra $\mathfrak g_{\le N}$ consists of primitive elements \cite{Reutenauer1993FreeLieAlgebras}. On $G_{\le N}$, the tensor exponential and logarithm are mutually inverse, giving a bijection
$$
\exp:\mathfrak g_{\le N}\to G_{\le N},\qquad \log:G_{\le N}\to \mathfrak g_{\le N},
$$
and hence a one-to-one correspondence between the signature and its log-signature. Since $\dim(\mathfrak g_{\le N}) \leq \dim(T_{\le N}(\R^d))$, the log-signature provides a compressed coordinate system.

Although the log-signature contains the same information as the signature, universal approximation results for signatures do not carry over to the log-signature. This is consistent with empirical findings that signatures often outperform log-signatures \cite{morrill2020generalised}. Nonetheless, log-signatures remain useful in a range of applications, owing to their lower-dimensional, information-dense representation \cite{liao2021logsigrnn,patil2024logsig}. 

For log-signature computations, \texttt{pathsig} supports the computationally efficient Lie basis introduced in \cite{kidger2020signatory}, which is naturally indexed by Lyndon words.  We compute this by taking the tensor logarithm of the signature. With support for alternative projections (Section~\ref{sec:projections}), we can compute and differentiate these log-signature coordinates without materialising all signature coefficients up to the target depth $N$. Instead, we compute all coefficients up to level $N-1$, and at level $N$ only the required subset. 
Since the number of coefficients at level $n$ grows like $d^n$, the highest level typically dominates runtime, so this yields substantial savings.

In principle, one could convert the log-signature to a different, more classical Lie basis, such as a Hall basis or a Lyndon bracket basis. A straight forward way to do this is to compute the log-signature in the expanded tensor basis and then apply a precomputed linear map to obtain the corresponding Lie coordinates. However, this is computationally expensive and introduces additional numerical error. For most machine-learning applications, this extra work is unnecessary, since different bases of the same Lie space are related by an invertible linear transformation and downstream linear layers can learn an equivalent change of coordinates. Moreover, for such Lie bases it is often preferable to compute log-signatures directly in that basis, for which more mathematically oriented libraries such as \texttt{RoughPy} \cite{morley2024roughpy} provide good support.

\section{Backpropagation through the Signature}
\label{sec:backprop}
When the signature is used as a trainable component of a machine-learning model~\cite{deepsig2019, sigtrans} rather than a fixed feature map, efficient backpropagation through the signature is essential for gradient-based training. In either case, a signature-based model can be written as
$$
X \mapsto f_\theta(\sig_{0,T}(X)),
$$
where $f_\theta$ represents the downstream part of the model and $X$ is the input path, which may take the form $X= \phi_\theta(Y)$ for a learnable map $\phi_\theta$ applied to an underlying path $Y$.

In modern machine learning frameworks, such as PyTorch, layers are self-contained modules that can be composed into larger models by forming a computation graph. During backpropagation, each layer in the computational graph receives gradients with respect to its outputs and must compute gradients with respect to its inputs. For the signature layer, the outputs are the signature coordinates
$$
\bigl\{\sig_{0,T}(X,w)\bigr\}_{w\in\mathcal I}
$$
indexed by a finite set of words $\mathcal I \subset \W$; e.g. $\mathcal I=\W_{\le N}$ for truncation at depth $N$. Accordingly, in the backward pass the signature layer receives the gradients
$$
\left\{\frac{\partial \loss}{\partial \sig_{0,T}(X,w)}\right\}_{w\in\mathcal I}
$$
from the subsequent part of the model and must produce the gradients with respect to the input path,
\begin{equation}\label{eq:path_grads}
\left\{\frac{\partial \loss}{\partial X_{j}^{(i)}}\right\}_{(i,j)\in \mathcal{A}_d \times \{0,1,\dots,M\}}.
\end{equation}
Here $X_{j}^{(i)}$ denotes the $i$-th coordinate of $X$ at time $t_j$, and $\loss=\mathcal{L}\left(f_\theta\left(\sig_{0,T}(X)\right), y\right)$ is the loss to be minimised.

\subsection{Gradients with Respect to Path Increments}
Since $\sig_{0,T}(X)$ depends on the path only through the increments $\Delta X_j=X_j-X_{j-1}$, we first compute the gradients $\left\{\partial \loss / \partial \Delta X_j\right\}_{j=1}^M$. The gradients with respect to the sampled path values in \eqref{eq:path_grads} then follow directly from the chain rule.

\begin{proposition}\label{thm:grad_decomp_formula}
Let $\mathcal{I} \subseteq \W$ be a non-empty set of words over the alphabet $\mathcal{A}_d$. For each pair $(i,j)\in \mathcal{A}_d \times \{1,\dots,M\}$,
$$
\frac{\partial \loss}{\partial \Delta X_j^{(i)}} =
\sum_{w \in \mathcal{I}} \frac{\partial \loss}{\partial \sig_{0,T}(X, w)} \sum_{\substack{u, v \in \W \\ u \circ v = w}} \frac{\partial \sig_{0,t_j}(X, u)}{\partial \Delta X_j^{(i)}} \, \sig_{t_j,T}(X, v).
$$
\end{proposition}
\begin{proof}
By the chain rule,
$$
\frac{\partial \loss}{\partial \Delta X_j^{(i)}} \;=\; \sum_{w \in \mathcal{I}}
\frac{\partial \loss}{\partial \sig_{0,T}(X, w)} \,
\frac{\partial \sig_{0,T}(X, w)}{\partial \Delta X_j^{(i)}}.
$$
We can isolate the dependence on $\Delta X_j$ in $\sig_{0,T}(X)$ using Chen's relation,
$$
\sig_{0,T}(X)
= \sig_{0,t_{j-1}}(X)\otimes \exp(\Delta X_j)\otimes \sig_{t_j,T}(X),
$$
where the path increment $\Delta X_j$ enters through
$\sig_{0,t_j}(X) = \sig_{0,t_{j-1}}(X) \otimes \exp(\Delta X_j)$.
Consequently, differentiation with respect to $\Delta X_j^{(i)}$ only acts on $\sig_{0,t_j}(X)$ in
$$
\sig_{0,T}(X)
= \sig_{0,t_{j}}(X)\otimes \sig_{t_j,T}(X).
$$
Projection onto the word set $\mathcal{I}$ gives the stated expression.
\end{proof}
A primary challenge is to evaluate this without storing the intermediate signatures $\sig_{0,t_j}(X)$ or $\sig_{t_j,T}(X)$ for all $j=1, \ldots, M-1$. Retaining all of these intermediate signatures leads to a memory cost that scales linearly with the number of time steps, which quickly becomes prohibitive. Instead, we reconstruct the required terms during the backward pass.

\subsection{Reconstructing Intermediate Signature Values}
To reconstruct the signature values required in the backward pass while storing only the terminal signature $\sig_{0,T}(X)$ from the forward pass, we rely on the algebraic structure of the signature. Namely, that the signature is \emph{group-like} and hence invertible with respect to $\otimes$. This allows us to express the required $\sig_{0,t_j}(X, \cdot)$ and $\sig_{t_j,T}(X, \cdot)$ values in terms of a minimal set of computable values.

\begin{proposition}\label{prop:rev_end}
For each $j \in \{0, \ldots, M-1\}$ and each $w \in \W$,
$$
\sig_{t_j, T}(X, w) = \sum_{w = w_p \circ w_s} \exp(\Delta X_{j+1}, w_p) \, \sig_{t_{j+1}, T}(X, w_s).
$$
\end{proposition}

\begin{proof}
This follows from Chen's relation:
$\sig_{t_j, T}(X) = \sig_{t_j, t_{j+1}}(X) \otimes \sig_{t_{j+1}, T}(X)$,
and that $\sig_{t_j,t_{j+1}}(X)=\exp(\Delta X_{j+1})$ by Proposition~\ref{prop:sig_over_sub}.
\end{proof}

Just as the signature can be computed over prefix-closed sets of words forward in time, this can be computed over suffix-closed sets of words backward in time. 

\begin{definition}
Let $w \in \W$. The set of all suffixes of $w$ is 
$$
S_w := \{ s \in \W \mid \exists u \in \W,\, w = u \circ s \}.
$$
\end{definition}

For the $\sig_{0,t_j}(X, \cdot)$ terms, we use a reversal property of the signature that allows us to compute it backward in time over prefix-closed sets of words. We recall the usual notion of a time-reversed path.

\begin{definition}
The \emph{time-reversed path} over $[0,t]$, denoted $\overleftarrow{X}$, is the path parameterised by $s \mapsto X_{t-s}$ for $s \in [0,t]$.
\end{definition}

\begin{lemma}
\label{lem:signature-inverse}
The inverse of $\sig_{0,t}(X)$ with respect to $\otimes$ is the signature of the time-reversed path:
$$
\sig_{0,t}(X)^{-1} = \sig_{0,t}(\overleftarrow{X}).
$$
\end{lemma}

\begin{proof}
Since $\sig_{0,t}(X) \otimes \sig_{0,t}(\overleftarrow{X}) = 1$, see for example \cite[Proposition 7.12]{friz2010multidimensional}, it follows from the group-like property of the signature that $\sig_{0,t}(X)^{-1} = \sig_{0,t}(\overleftarrow{X}).$
\end{proof}

\begin{proposition}\label{thm:rev_word_sig}
For each $j \in \{0, \ldots, M-1\}$ and each $w \in \W$,
$$
\sig_{0, t_j}(X, w) = \sum_{w = w_p \circ w_s} \sig_{0, t_{j+1}}(X, w_p) \, \exp(-\Delta X_{j+1}, w_s).
$$
\end{proposition}

\begin{proof}
Since $\sig_{t_j, t_{j+1}}(X) = \exp(\Delta X_{j+1})$, the inverse is $\exp(-\Delta X_{j+1})$ according to Lemma~\ref{lem:signature-inverse}. Thus, $\sig_{0, t_j}(X) = \sig_{0, t_{j+1}}(X) \otimes \exp(-\Delta X_{j+1})$.
\end{proof}

Propositions~\ref{prop:rev_end} and \ref{thm:rev_word_sig} show that, for a fixed word $w$, its contribution to the path gradients in \eqref{eq:path_grads} depends only on signature coordinates indexed by prefixes and suffixes of $w$. The corresponding values,
$$
\{\sig_{0,t_j}(X,u)\}_{u\in P_w}
\qquad\text{and}\qquad
\{\sig_{t_j,T}(X,v)\}_{v\in S_w},
$$
can be reconstructed via backward recursion in $j$, with terminal conditions $\sig_{0,t_M}(X)=\sig_{0,T}(X)$ and $\sig_{t_M,T}(X)=1$. Consequently, backpropagation can be carried out over arbitrary word sets $\mathcal{I}\subset\W$, while storing only the terminal signature $\sig_{0,T}(X)$.

\FloatBarrier

\section{Signatures over Windows}\label{sec:windowed-signatures}
It is often advantageous to represent a path not by a single global signature $\mathbf{S}_{0, T}(X)$, but by a collection of signatures over multiple subintervals,
$$
\bigl(\mathbf{S}_{a,b}(X)\bigr)_{(a,b)\in I},
\qquad
I \subseteq \{(a,b)\in[0,T]^2 : a < b \}.
$$
Such a collection can be interpreted from a time-series perspective by choosing a collection of windows $I$ and treating the resulting signatures as a structured feature set, as in \cite{morrill2020generalised}. Another viewpoint, introduced in \cite{cuchiero2023signature}, is to treat signatures (or linear functionals thereof) as stochastic processes, for example
$t\mapsto \ell(\mathbf{S}_{0,t}(X))$, corresponding to expanding windows, or $t\mapsto \ell(\mathbf{S}_{t-h,t}(X))$ for a fixed horizon $h$, corresponding to sliding windows.

Signature libraries often provide limited native support for such computations, typically only for expanding windows with unit stride, and otherwise requiring a separate evaluation per window. In contrast, \texttt{pathsig} supports general, user-specified windows. Concretely, suppose $X$ is observed at times $0=t_0<\cdots<t_M=T$, and let $(\ell_i,r_i)_{i=1}^K$ be index pairs with $\ell_i<r_i$ specifying windows $[t_{\ell_i},t_{r_i})$.
Given the resulting $K\times 2$ integer tensor of indices, \texttt{pathsig} returns the collection
$$
\bigl(\mathbf{S}_{t_{\ell_i},t_{r_i}}(X)\bigr)_{i=1}^K
$$
in a single evaluation.

This reduces the fixed overheads that often dominate windowed signature computations. Since \texttt{pathsig} evaluates an entire collection of windows in a single call, we incur essentially the same fixed overhead as for a single signature computation. Moreover, windowing adds a further axis of parallelism beyond the usual $(\text{paths},\text{words})$. For small batch sizes and moderate depths, parallelising only over the paths in the batch and words often leaves the GPU under-utilised. Introducing a window dimension increases the total number of independent signature computations, so the CUDA kernel can expose enough parallel work to saturate the device. We could have used Chen’s relation to recover each windowed signature as
$\mathbf{S}_{t_\ell,t_r}(X)=\mathbf{S}_{0,t_\ell}(X)^{-1}\otimes \mathbf{S}_{0,t_r}(X)$
from precomputed expanding-window signatures, as in \texttt{Signatory} \cite{kidger2020signatory}, but this can be numerically unstable and memory-intensive for long sequences, and typically offers marginal benefit unless the windows have a high degree of overlap.
\section{Benchmarks}
In our benchmarks, we use \texttt{keras\_sig} \cite{genet2025keras} and \texttt{pySigLib} \cite{shmelev2025pysiglib} as references, since both have been shown to deliver comparable or superior performance to previous libraries. All experiments were conducted on a single NVIDIA H200 GPU with 140 GB of VRAM. The host system was an Intel Xeon Platinum 856-series machine with 24 allocated CPU cores and 256 GB of RAM. We note, however, that similar results were observed on consumer-grade GPUs.

We evaluate performance across three primary metrics: runtime, memory usage, and scaling behavior. To ensure reliability, each benchmark included 3--10 warm-up runs, followed by averaged measurements over 10--50 runs. We report results for both forward-only execution and training. Here, training refers to a single training step consisting of one forward pass and its corresponding backward pass.

\begin{remark}
\texttt{pySigLib} computes signatures and log-signatures on the CPU, noting that for their algorithms this can be as fast as, or faster than, a GPU implementation \cite{pysiglibdocs}. This choice can be sensible when signatures are used as a preprocessing step or when their cost is a small fraction of the overall workload. In large-scale multi-GPU training, however, placing a core part of the model’s computation on the host tends to scale poorly. It competes with CPU resources needed for data loading and preprocessing. Moreover, CPU shared-memory parallelism, as used by \texttt{pySigLib}, typically saturates at modest thread counts, so the host is likely to become the bottleneck as GPU throughput increases.
\end{remark}

\subsection{Truncated Signature Performance}
Overall, \texttt{pathsig} consistently outperforms \texttt{keras\_sig} and \texttt{pySigLib}, with median speedups of $12.44\times$ and $40.11\times$ for signature computation, and median training speedups of $7.88\times$ and $24.88\times$, respectively. Figure~\ref{fig:pathsig-keras-speedup} summarizes the results for 775 configurations with $d \in \{3,4,6,8,10,20,30,40\}$ and $N \in \{2,3,4,5,6\}$. The resulting truncated signature sizes range from $12$ to $299,600$. 

\begin{figure}[!htbp]
  \centering
  \includegraphics[width=0.95\linewidth]{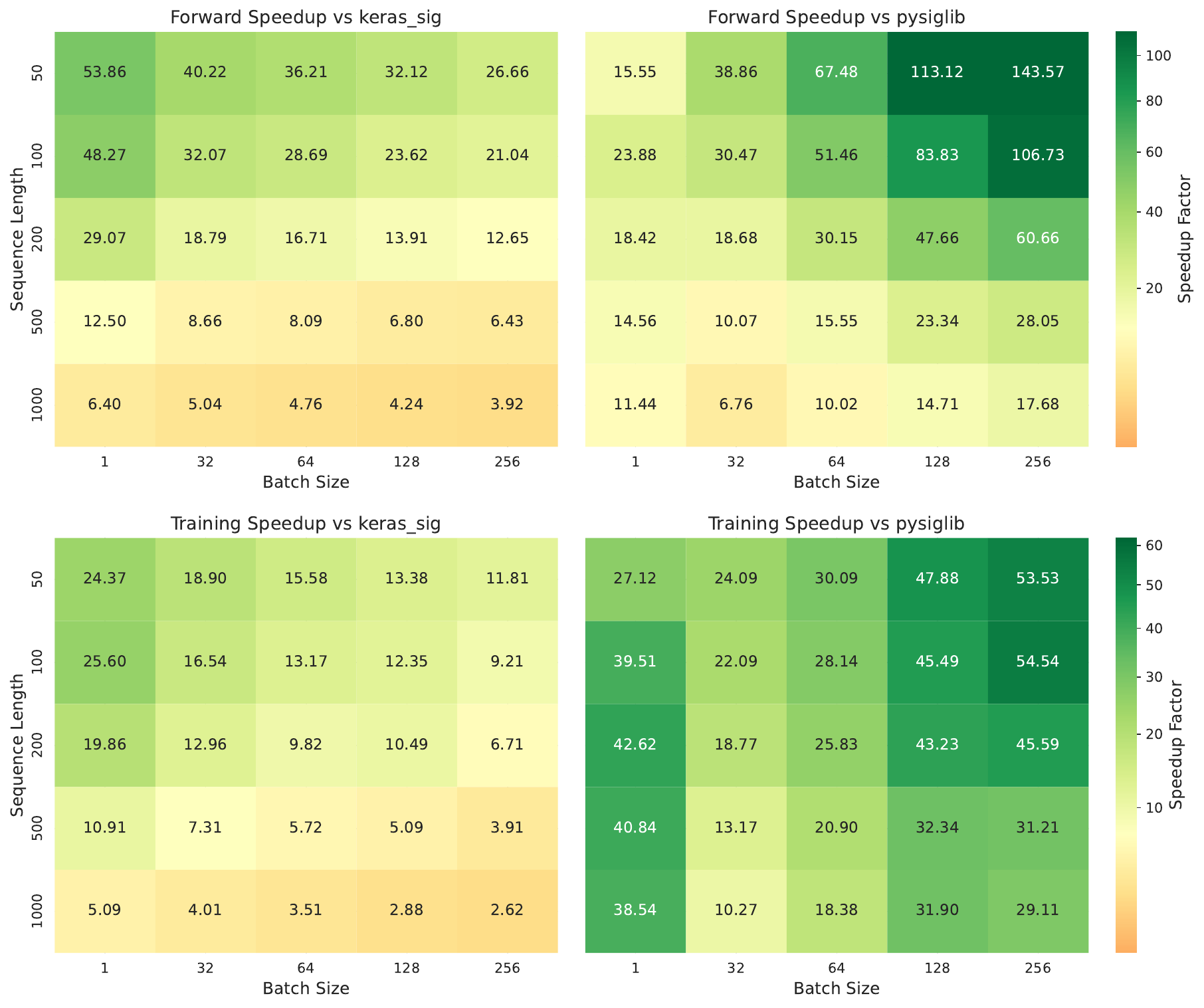}
  \caption{Speedup of \texttt{pathsig} relative to \texttt{keras\_sig} and \texttt{pySigLib}, averaged over 27 signature configurations for each combination of batch size and sequence length.}
  \label{fig:pathsig-keras-speedup}
\end{figure}
\FloatBarrier

Among the benchmarked configurations, \texttt{pathsig} improves with batch size and signature size until the GPU is saturated, after which runtime scales approximately linearly in both. Because \texttt{pathsig} does not parallelise over sequence length (unlike \texttt{keras\_sig}), relative speedups tend to be largest when sequence lengths are small and diminish as sequence length increases. Nevertheless, in the training benchmark, slowdowns occur only in three cases at $(B, M)=(1,1000)$ and depth $N=2$, with a minimum speedup of $0.94 \times$. Table~\ref{tab:speedup_train_selected} reports representative results.

\begin{table}[!htbp]
  \centering
  \footnotesize
  \setlength{\tabcolsep}{4pt}
  \caption{Selected training-time speedups of \texttt{pathsig} relative to \texttt{keras\_sig} and \texttt{pySigLib}.}
  \label{tab:speedup_train_selected}
  \begin{tabular}{llr rrr rr}
    \toprule
    & & & \multicolumn{3}{c}{Training time} & \multicolumn{2}{c}{Speedup of \texttt{pathsig}} \\
    \cmidrule(lr){4-6}\cmidrule(lr){7-8}
    $(B,M,d)$ & $N$ & Sig dim
    & \texttt{keras\_sig}
    & \texttt{pySigLib}
    & \texttt{pathsig}
    & vs \texttt{keras\_sig}
    & vs \texttt{pySigLib} \\
    \midrule

    \multicolumn{8}{l}{\textit{Effect of depth} $N$}\\
    (32, 100, 6) & 2 & 42    & 2.08 ms  & 3.88 ms  & 151.29 $\mu$s & $13.77\times$ & $25.65\times$ \\
    (32, 100, 6) & 3 & 258   & 5.84 ms  & 3.83 ms  & 270.70 $\mu$s & $21.57\times$ & $14.16\times$ \\
    (32, 100, 6) & 4 & 1.6K  & 4.56 ms  & 9.12 ms  & 207.52 $\mu$s & $21.97\times$ & $43.96\times$ \\
    (32, 100, 6) & 5 & 9.3K  & 8.65 ms  & 23.86 ms & 555.63 $\mu$s & $15.57\times$ & $42.95\times$ \\
    (32, 100, 6) & 6 & 56.0K & 18.25 ms & 64.60 ms & 3.36 ms       & $5.43\times$  & $19.22\times$ \\
    \midrule

    \multicolumn{8}{l}{\textit{Effect of seq length} $M$}\\
    (64,  50, 4) & 6 & 5.5K  & 9.62 ms  & 18.59 ms  & 351.67 $\mu$s & $27.35\times$ & $52.87\times$ \\
    (64, 100, 4) & 6 & 5.5K  & 17.73 ms & 31.66 ms  & 673.91 $\mu$s & $26.31\times$ & $46.98\times$ \\
    (64, 200, 4) & 6 & 5.5K  & 17.68 ms & 44.68 ms  & 1.31 ms       & $13.49\times$ & $34.11\times$ \\
    (64, 500, 4) & 6 & 5.5K  & 19.50 ms & 87.35 ms  & 3.25 ms       & $5.99\times$  & $26.85\times$ \\
    (64,1000, 4) & 6 & 5.5K  & 17.51 ms & 156.85 ms & 6.45 ms       & $2.72\times$  & $24.33\times$ \\
    \midrule

    \multicolumn{8}{l}{\textit{Effect of batch size} $B$}\\
    (  1, 200, 10) & 4 & 11.1K & 8.99 ms  & 16.07 ms  & 320.35 $\mu$s & $28.07\times$ & $50.15\times$ \\
    ( 32, 200, 10) & 4 & 11.1K & 9.53 ms  & 32.94 ms  & 1.36 ms       & $6.98\times$  & $24.14\times$ \\
    ( 64, 200, 10) & 4 & 11.1K & 9.19 ms  & 46.98 ms  & 2.59 ms       & $3.55\times$  & $18.14\times$ \\
    (128, 200, 10) & 4 & 11.1K & 10.27 ms & 161.84 ms & 5.01 ms       & $2.05\times$  & $32.31\times$ \\
    (256, 200, 10) & 4 & 11.1K & 19.79 ms & 235.71 ms & 9.83 ms       & $2.01\times$  & $23.97\times$ \\
    \bottomrule
  \end{tabular}
\end{table}
\FloatBarrier

\subsection{Memory Usage}
Memory usage, a frequent bottleneck in signature methods, remains exceptionally low with \texttt{pathsig}. In single precision (\texttt{float32}), the theoretical minimum memory required to store only the signature output is
$$
  \mathrm{Mem}_{\mathrm{out}} \approx 4\,B\,D_{\text{sig}} \ \text{bytes},
$$
where $D_{\text{sig}} := \sum_{n=1}^{N} d^n$ is the dimension of the truncated signature (excluding level 0) and $B$ is the batch size (number of paths). Even at the most challenging benchmarked configurations, peak memory remains modest for \texttt{pathsig}, with it staying near $\approx 2\times$ the memory required for the signature output.

Table~\ref{tab:vram_h200_fp32_selected} reports selected peak VRAM allocations during training for \texttt{pathsig} and \texttt{keras\_sig}. Since \texttt{pySigLib} runs on the CPU, its memory usage is in host RAM rather than VRAM, so a direct comparison is not meaningful and we omit it here. We note, however, that both \texttt{pathsig} and \texttt{pySigLib} have memory footprints that scale as $O(BD_{\text{sig}})$, whereas \texttt{keras\_sig} scales as $O(BMD_{\text{sig}})$. Consequently, several \texttt{keras\_sig} configurations terminate with OOM (out-of-memory) errors despite the 140\,GB of VRAM available on the H200.

\begin{table}[!htbp]
  \centering
  \footnotesize
  \setlength{\tabcolsep}{5pt}
  \caption{Selected peak VRAM allocations during training.}
  \label{tab:vram_h200_fp32_selected}
  \begin{tabular}{llr rrrrr}
    \toprule
    $(B,M,d)$ & $N$ & Sig dim
    & $\mathrm{Mem}_{\mathrm{out}}$
    & \texttt{keras\_sig}
    & \texttt{pathsig}
    & Reduction ($\times$) \\
    \midrule
    \multicolumn{7}{l}{\textit{Effect of depth} $N$}\\
    (32, 50, 8) & 3 & 584 & 0.1 MB  & 24.2 MB   & 0.3 MB  & $81\times$ \\
    (32, 50, 8) & 4 & 4.7K & 0.6 MB  & 245.5 MB  & 1.3 MB  & $189\times$ \\
    (32, 50, 8) & 5 & 37.4K & 4.6 MB  & 1.53 GB   & 9.3 MB  & $165\times$ \\
    (32, 50, 8) & 6 & 299.6K & 36.6 MB & 11.64 GB  & 73.3 MB & $159\times$ \\
    \midrule
    \multicolumn{7}{l}{\textit{Effect of seq length} $M$}\\
    (32,  50, 8) & 6 & 299.6K & 36.6 MB & 11.64 GB   & 73.3 MB & $159\times$ \\
    (32, 100, 8) & 6 & 299.6K & 36.7 MB & 23.48 GB   & 73.5 MB & $319\times$ \\
    (32, 200, 8) & 6 & 299.6K & 36.8 MB & 47.16 GB   & 73.9 MB & $638\times$ \\
    (32, 400, 8) & 6 & 299.6K & 37.0 MB & 94.51 GB   & 74.7 MB & $1{,}265\times$ \\
    (32, 800, 8) & 6 & 299.6K & 37.4 MB & \textsc{OOM} & 76.3 MB & --- \\
    (32,1600, 8) & 6 & 299.6K & 38.1 MB & \textsc{OOM} & 79.4 MB & --- \\
    \midrule
    \multicolumn{7}{l}{\textit{Effect of batch size} $B$}\\
    ( 32, 50, 8) & 6 & 299.6K & 36.6 MB  & 11.64 GB & 73.3 MB  & $159\times$ \\
    ( 64, 50, 8) & 6 & 299.6K & 73.2 MB  & 23.27 GB & 146.7 MB & $159\times$ \\
    (128, 50, 8) & 6 & 299.6K & 146.5 MB & 46.55 GB & 293.3 MB & $159\times$ \\
    (256, 50, 8) & 6 & 299.6K & 293.0 MB & 93.10 GB & 586.7 MB & $159\times$ \\
    \bottomrule
  \end{tabular}
\end{table}
\FloatBarrier

\subsection{Log-Signature Performance}
For our log-signature implementation, we benchmark against \texttt{pySigLib}. Using the same configurations as in the truncated signature benchmark, Figure~\ref{fig:pathsig-log-speedup} shows that \texttt{pathsig} achieves even larger speedups. 

\begin{figure}[!htbp]
  \includegraphics[width=\linewidth]{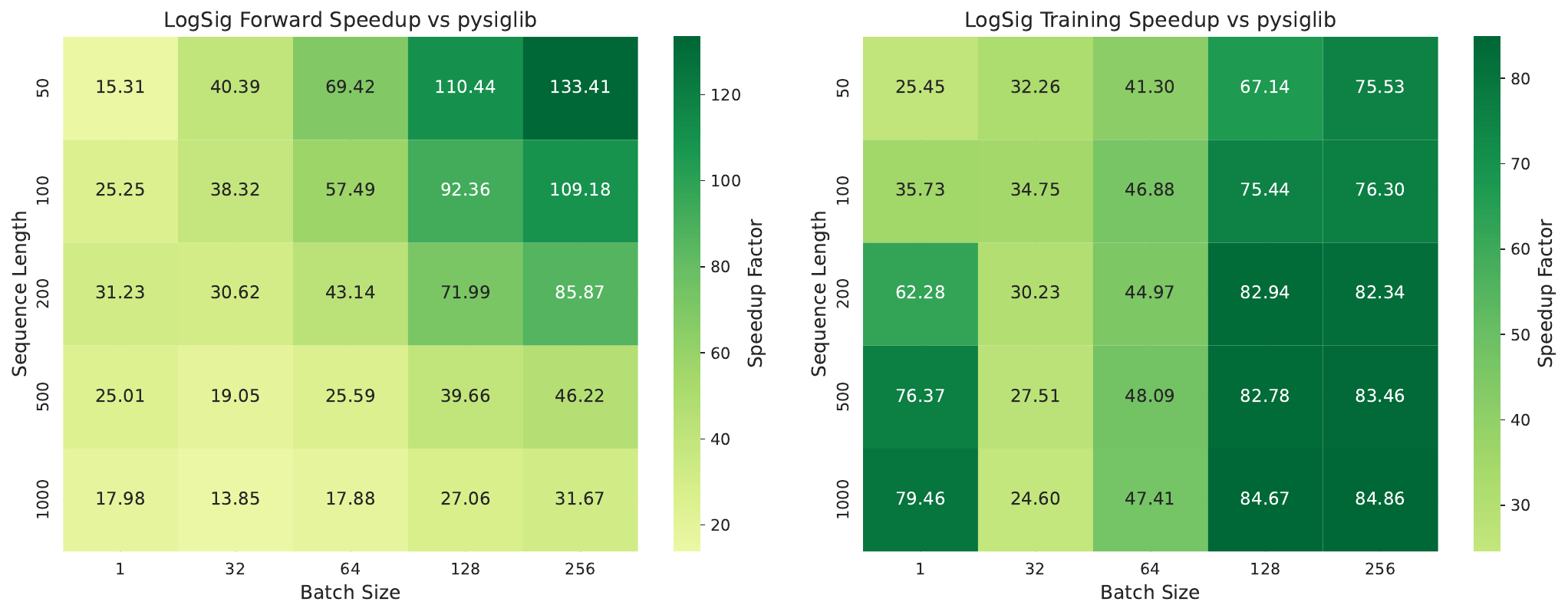}
  \caption{Speedup of \texttt{pathsig} relative to \texttt{pySigLib}, averaged over 27 log-signature configurations for each combination of batch size and sequence length.}
  \label{fig:pathsig-log-speedup}
\end{figure}
\FloatBarrier

The larger speedups stem from \texttt{pathsig} using an alternative projection that avoid materialising all signature coordinates up to the truncation depth. As a result, log-signature computation is often $2$--$3\times$ faster than the corresponding signature computation in \texttt{pathsig} (Table~\ref{tab:speedup_logsig_train_selected}).

\begin{table}[!htbp]
  \centering
  \footnotesize
  \setlength{\tabcolsep}{4pt}
  \caption{Selected log-signature training-time speedups of \texttt{pathsig} relative to \texttt{pySigLib}.}
  \label{tab:speedup_logsig_train_selected}
  \begin{tabular}{llrrrrr}
    \toprule
    $(B,M,d)$ & $N$
    & LogSig dim
    & \texttt{pySigLib}
    & \texttt{pathsig}
    & Speedup ($\times$) \\
    \midrule

    \multicolumn{6}{l}{\textit{Effect of depth} $N$}\\
    (32, 100, 6) & 3 & 91   & 5.92 ms  & 321.22 $\mu$s & $18.43\times$ \\
    (32, 100, 6) & 4 & 406  & 12.14 ms & 554.58 $\mu$s & $21.89\times$ \\
    (32, 100, 6) & 5 & 2.0K & 29.24 ms & 430.57 $\mu$s & $67.91\times$ \\
    (32, 100, 6) & 6 & 9.7K & 80.61 ms & 1.38 ms       & $58.34\times$ \\
    \midrule

    \multicolumn{6}{l}{\textit{Effect of seq length} $M$}\\
    (64,  50, 4) & 6 & 964 & 22.81 ms  & 652.84 $\mu$s & $34.94\times$ \\
    (64, 100, 4) & 6 & 964 & 37.46 ms  & 497.20 $\mu$s & $75.35\times$ \\
    (64, 200, 4) & 6 & 964 & 41.32 ms  & 806.30 $\mu$s & $51.25\times$ \\
    (64, 500, 4) & 6 & 964 & 89.96 ms  & 1.69 ms       & $53.08\times$ \\
    (64,1000, 4) & 6 & 964 & 170.48 ms & 3.18 ms       & $53.64\times$ \\
    \midrule

    \multicolumn{6}{l}{\textit{Effect of batch size} $B$}\\
    (  1, 200, 10) & 4 & 2.9K & 20.89 ms  & 601.12 $\mu$s & $34.75\times$ \\
    ( 32, 200, 10) & 4 & 2.9K & 40.73 ms  & 629.47 $\mu$s & $64.70\times$ \\
    ( 64, 200, 10) & 4 & 2.9K & 47.47 ms  & 1.13 ms       & $42.13\times$ \\
    (128, 200, 10) & 4 & 2.9K & 110.66 ms & 1.89 ms       & $58.55\times$ \\
    (256, 200, 10) & 4 & 2.9K & 189.75 ms & 3.51 ms       & $54.14\times$ \\
    \bottomrule
  \end{tabular}
\end{table}
\FloatBarrier

\subsection{Windowed Signature Performance}
We benchmark windowed signature computation against \texttt{pySigLib}, using the same configurations as in our other benchmarks and restricting to cases with signature size below $10,000$. For each signature configuration, we benchmarked across batch sizes of 1, 16, and 32, along with window lengths from 4 to 128 and number of windows from 2 to 1024, both varying in powers of 2.

Across the resulting 2,700 configurations, we observed speedups ranging from $3.92\times$ to $6,380\times$, with a median of $153\times$ and a mean of $395\times$. The benchmark shows that the speedup for \texttt{pathsig} increases until all GPU resources are exhausted, after which the speedup converges to the two- to three-digit range. Figure~\ref{fig:win-train-scaling} shows the scaling with number of windows for different batch sizes.

\begin{figure}[!htbp]
  \centering
  \includegraphics[width=0.49\linewidth]{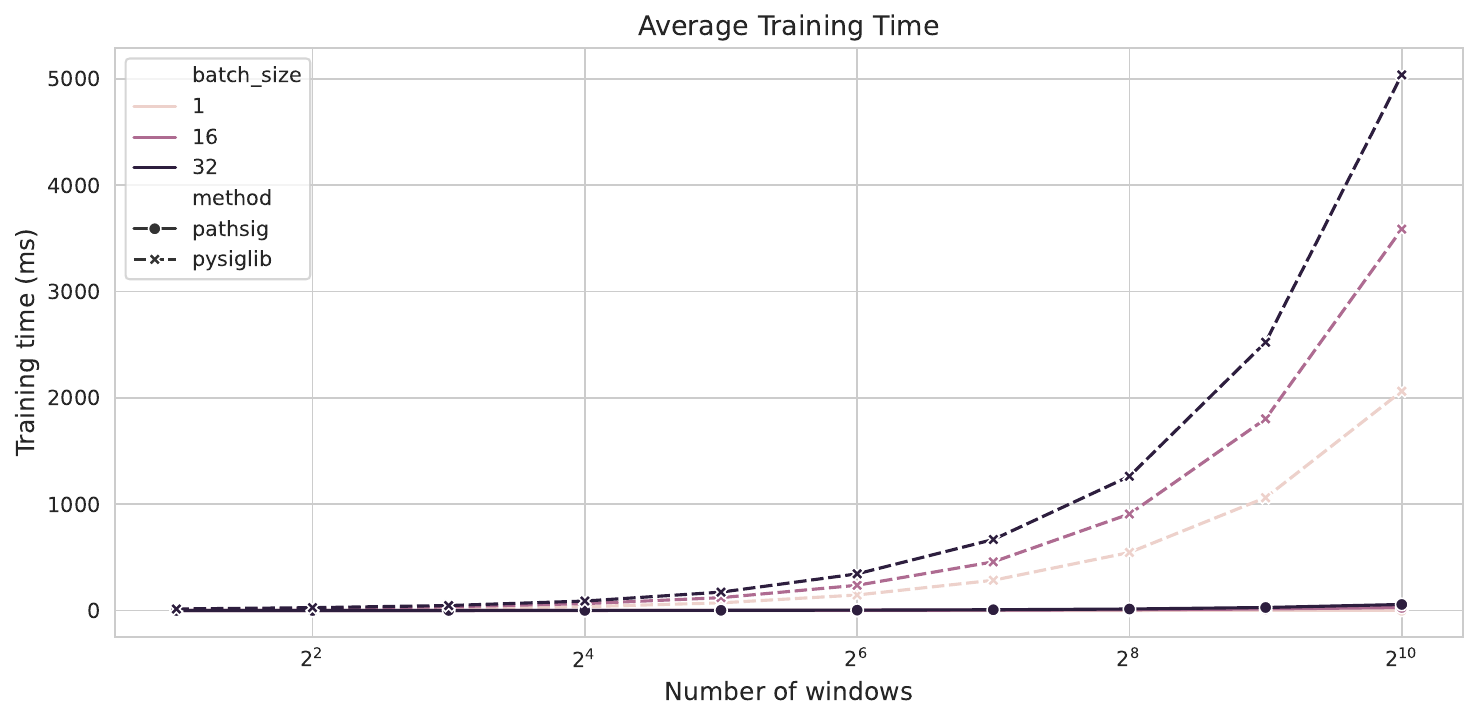}\hfill
  \includegraphics[width=0.49\linewidth]{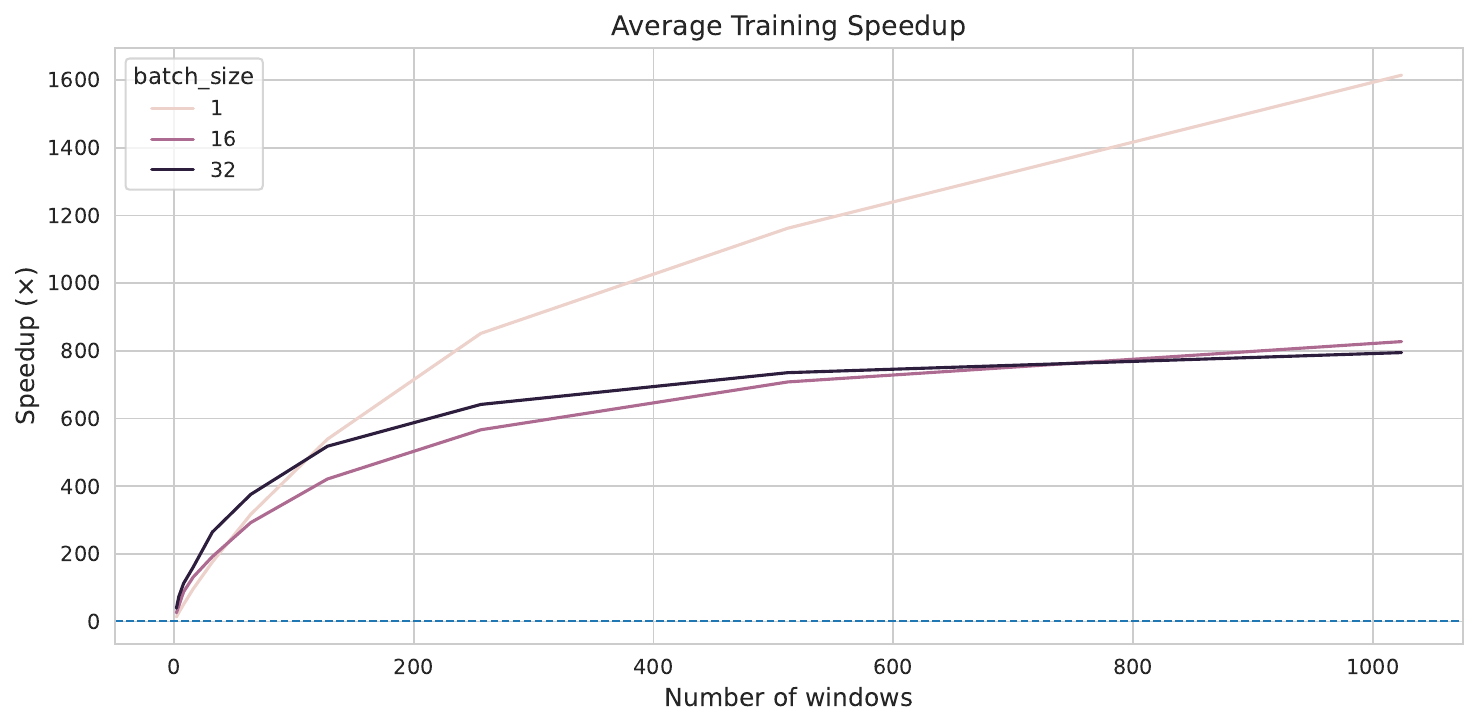}
  \caption{Average training time (left) and corresponding speedup (right) for windowed signature computation.}
  \label{fig:win-train-scaling}
\end{figure}
\FloatBarrier

\section{Signature Projections}
\label{sec:projections}
The signature is infinite-dimensional, so any practical use requires selecting a finite set of coordinates. The canonical choice is truncation at depth $N$,
$$
\pi_{\le N}\bigl(\sig_{0,T}(X)\bigr)
=\bigl(1,\sig^{(1)}_{0,T}(X),\ldots,\sig^{(N)}_{0,T}(X)\bigr)\in T_{\le N}(\mathbb R^d),
$$
which preserves the graded algebraic structure up to degree $N$ and thus permits computation using only tensors in $(\mathbb R^d)^{\otimes n}$ for $n\le N$. For machine-learning tasks, however, computing all coefficients up to depth $N$ is rarely intrinsic. Universal approximation results ensure that signature features are expressive, but they do not imply that $\W_{\le N}$ is the best coordinate choice at a fixed budget. In fact, when one fits linear functionals on truncated signatures with $\ell^1$-regularisation, the selected coefficients are often sparse \cite{guo2023lasso}.

Because our computational approach extends to arbitrary finite word sets, it is not restricted to truncation. This enables general projections with minimal structural assumptions. Since an appropriate projection is inherently task-dependent, we treat this as a modelling decision rather than prescribing a particular scheme. The resulting projected signatures remain fully differentiable and can be used directly in gradient-based training. Such projections can yield more compact representations and permit inclusion of selected higher-order terms without incurring the cost of full truncation at the corresponding depth. 

\subsection{Word Projections}
The most general projection supported by \texttt{pathsig} is onto an arbitrary prescribed set of words. For a non-empty $\mathcal I\subset\mathcal W$, this projection is given by
$$
\pi_\mathcal I\bigl(\sig_{0,T}(X)\bigr):=\bigl\{\sig_{0,T}(X,w)\bigr\}_{w\in\mathcal I} \in\mathbb R^{|\mathcal I|}.
$$

The selection of a word projection may be guided by computational budget, structure, or heuristics. In some applications, the target functional is known (or hypothesised) to depend on specific words or prescribed linear combinations of signature coefficients. One such example is signature payoffs \cite{arribas2018signaturepayoffs}, where even derivatives with complex path dependence often can be described by relatively simple payoff functionals. As a result, often only a small set of signature coordinates, indexed by particular words, is needed to represent the payoff. Even when only a small collection of signature terms is retained, one can often recover strong expressive power by allowing richer function classes on the retained terms. In particular, polynomials in finitely many signature terms enjoy a universal approximation property for functionals of non-geometric rough paths \cite{harang2024universal}. Similarly, the signature coefficients of Lyndon words form a minimal algebraic generating set (via shuffle relations) \cite{gaines1994algebra}, so a sufficiently nonlinear downstream model can, in principle, learn the relevant non-Lyndon signature coefficients from these generators. 

Other ways to construct word sets include imposing structural constraints or using variable-selection methods for high-dimensional feature spaces. A convenient approach is to encode prior knowledge about interactions between channels through a directed acyclic graph
$$
G=(A, E), \quad A=\{1, \ldots, d\}
$$
where an edge $(i, j) \in E$ indicates that channel $i$ can be followed by channel $j$. The induced word set
$$
\mathcal{W}_{\leq N}(G)=\left\{(i_1, \dots, i_n) : n \leq N, \left(i_k, i_{k+1}\right) \in E \text { for all } k=1, \ldots, n-1 \right\}
$$
has the advantage of being a hierarchical sparsification. This construction is particularly useful when channels have an underlying geometry as $E$ can be chosen to reflect local neighborhood structure or interaction constraints induced by a model (e.g., state-space couplings). 

In the machine-learning setting, word projections offer several extensions of signature-based methods. Each projection of the signature onto a prescribed word set is differentiable (Section~\ref{sec:backprop}) and can therefore be made learnable in the same spirit as deep signature models \cite{deepsig2019}. Such models apply a learnable transformation $\phi_\theta$ to $X$ and compute its truncated signature,
$$
X \longmapsto \pi_{\le N}\!\left(\sig_{0,T}\!\left(\phi_\theta(X)\right)\right),
$$
which is then passed to a downstream model that may include additional signature layers. However, using a single transformation $\phi_\theta$ for all retained signature coordinates up to level $N$ ties them to a single representation of the path and hence restricts the effective degrees of freedom. With word projections, we may instead partition $\mathcal I\subset\W$ into $\left\{\mathcal{I}_j\right\}_{j=1}^J$. By associating a separate learnable transformation $\phi_{\theta_j}^{(j)}$ to each subset of words, we can obtain a more expressive map
$$
X \longmapsto
\Bigl(
\pi_{\mathcal I_1}\bigl(\sig_{0,T}(\phi_{\theta_1}^{(1)}(X))\bigr),\;
\ldots,\;
\pi_{\mathcal I_J}\bigl(\sig_{0,T}(\phi_{\theta_J}^{(J)}(X))\bigr)
\Bigr)
\;\in\;
\R^{|\mathcal I_1|}\times\cdots\times \R^{|\mathcal I_J|}
\cong \R^{|\mathcal I|}.
$$
This allows different parts of the signature to adapt to different representations of the path. In similar fashion, one could also associate one or more learnable transformations to specific linear functionals on the signature.

\subsection{Anisotropic Signature}
The truncation level of the truncated signature is often chosen with the regularity of the underlying path $X$ in mind. However, because truncation is by word length, it does not distinguish between channels and therefore imposes uniform regularity across channels. The idea that a driving signal may be inhomogeneous across channels can be traced back to Lyons’ seminal paper \cite{lyons1998differential}. It has since been formalized as geometric $\Pi$-rough paths and anisotropic geometric rough paths \cite{gyurko2012pirough,tapiazambotti2020branched}. Likewise, we define the anisotropic signature by replacing word length by a weighted notion of degree.
\begin{definition}\label{def:inhom_trunc}
Fix weights $\gamma=(\gamma_1,\ldots,\gamma_d)\in\mathbb R_+^d$. For a word $w=(i_1,\ldots,i_n)\in\mathcal W$ define its weighted degree
$$
|w|_\gamma:=\gamma_{i_1}+\cdots+\gamma_{i_n},
\qquad
|\eps|_\gamma:=0.
$$
\end{definition}
Given a threshold $r\in\mathbb R_+$, we define the admissible set of words as
$$
\mathcal W^\gamma_{\le r}:=\{w\in\mathcal W:\ |w|_\gamma\le r\},
$$
and the anisotropic signature with cutoff $r$ by
$$
\mathcal S^{\gamma}_{\le r}(X)_{0,T}:=\bigl\{\mathcal S_{0,T}(X,w)\bigr\}_{w\in\mathcal W^\gamma_{\le r}}.
$$

As with truncation, Chen's relation and related algebraic properties remain after restricting to $\W^\gamma_{\le r}$. This can make the anisotropic signature a useful middle ground between full truncation and word projections, providing a structured way to reduce dimensionality when channel regularity is believed to be inhomogeneous. In application, the weights $\gamma$ may be treated as hyperparameters that control which channels contribute higher-order signature coordinates.

\section{Illustrative Example: Sparse Lead–Lag Signature Projection}
To demonstrate how alternative signature projections can reduce dimensionality while preserving predictive performance, we consider the problem of estimating the Hurst parameter $H$ of a multivariate fractional Brownian motion (fBM). For this we use a deep signature model similar to the one presented in \cite{deepsig2019}, with the exception that we consider a multivariate fBM and the lead-lag transformation. We recall the standard lead--lag transform (see, e.g., \cite{FlintHamblyLyons2016Hoff}).

\begin{definition}[Lead-lag transform]
Let $0=t_0<\dots<t_M=T$ and let $X:[0,T]\to\mathbb{R}^d$ be a path.
Define a sequence $(\widehat X_m)_{m=0}^{2M}$ in $\mathbb R^{2d}$ by
$$
\widehat X_{2k}   := (X_{t_k},\,X_{t_k}),\qquad
\widehat X_{2k+1} := (X_{t_k},\,X_{t_{k+1}}),\qquad k=0,\dots,M-1,
$$
and $\widehat X_{2M}:=(X_{t_M},X_{t_M})$.
The \emph{lead-lag path} $\widehat X:[0,T]\to\mathbb{R}^{2d}$ is the piecewise linear interpolation of the points
$\widehat X_0,\widehat X_1,\dots,\widehat X_{2M}$.
\end{definition}

An important property of the lead-lag path $\widehat{X}$ is that the antisymmetric part of its level-2 signature encodes the signed area between the lead and lag channels. In the semimartingale setting, this area coincides with the discrete quadratic (co)variation of the underlying path $X$, and it converges to the usual quadratic (co)variation as the mesh tends to zero. This yields the Itô-Stratonovich correction terms, so the ltô integral can be recovered from the lead-lag construction \cite[Theorem~5.1]{FlintHamblyLyons2016Hoff}. However, for this purpose the lead-lag transformation can be quite redundant. If $X$ is a $d$-dimensional continuous semimartingale with independent components, the only non-zero quadratic (co)variations are those of the form $\left[X^i, X^i\right]$ for $i=1, \ldots, d$. Hence many lead-lag signature terms that would encode quadratic covariation are identically trivial. 

For our experiment we consider a $d$-dimensional fractional Brownian motion $X=(X^1,\ldots,X^d)$ with independent components. Fractional Brownian motion is not a semimartingale when $H \neq \frac{1}{2}$, so the usual Itô calculus does not apply. Nevertheless, for $H>\frac{1}{4}$ it admits a geometric rough path lift \cite{friz2010multidimensional}, so second-level iterated integrals are well defined and the lead--lag transform captures the corresponding area type interactions. We label the coordinates of its
lead--lag transform by the alphabet
$$
\mathcal{A}_{LL}=\{\ell_1,\ldots,\ell_d,\,L_1,\ldots,L_d\},
$$
where $\ell_i$ denotes the lag channel and $L_i$ the lead channel for the $i$-th coordinate of $X$. Since the components are independent, our sparse word projection excludes words generated by $(\ell_i,L_j)$ or $(L_i,\ell_j)$ for $i\neq j$. Concretely, let
$$
\mathcal{G}=\left\{\eps, (L_1), \ldots, (L_d)\right\} \cup\left\{\left(\ell_i, L_i\right),\left(L_i, \ell_i\right): i=1, \ldots, d\right\},
$$
and define
$$
\mathcal{W}_{\leq N}^{\text {sparse }}=\left\{w=u_1 \circ \cdots \circ u_p: u_j \in \mathcal{G},|w| \leq N\right\}.
$$

In our experiments, this sparse word projection attains lower test error than the corresponding truncated signature and exhibits slightly improved learning curves (Figure~\ref{fig:lering}). At the same time, it reduces feature dimension by 6.25× and end-to-end training time by 2.24×.
\begin{figure}[!htbp]
  \centering
  \includegraphics[width=0.85\linewidth]{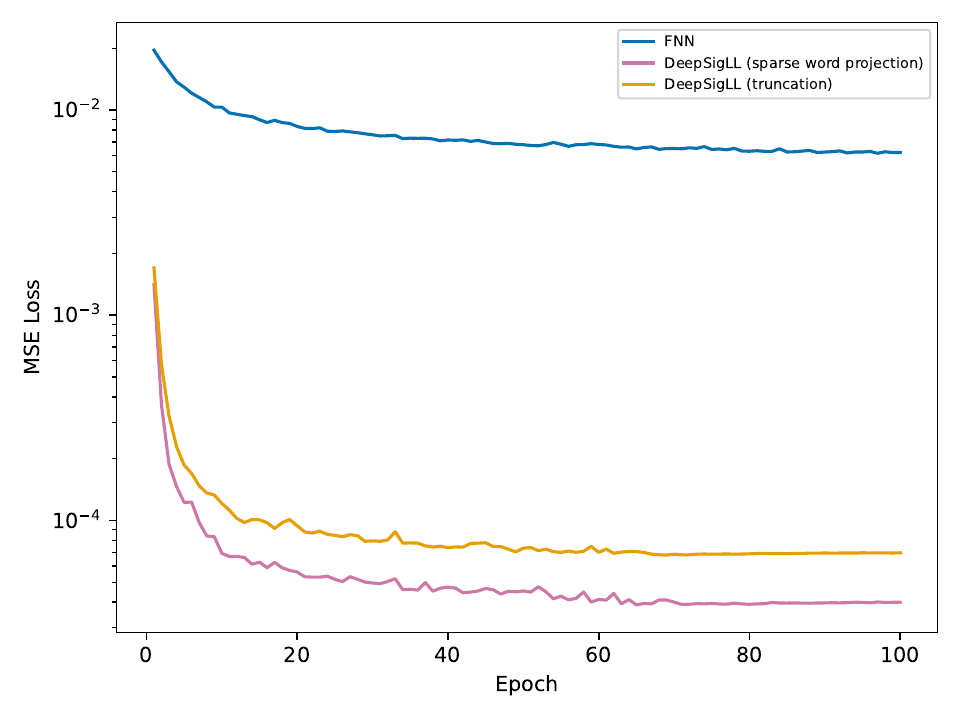}
  \caption{Validation MSE versus epoch (mean over 5 independent runs) for Hurst parameter estimation on $5$-dimensional fBM with independent components. We train on $8{,}000$ simulated paths of length $250$ with $H \sim U(0.25,0.75)$ sampled independently per path, and evaluate on $2{,}000$ held-out paths. We compare an FNN (Feedforward Neural Network) baseline to the Deep Signature Model using either a truncated lead--lag signature or the proposed sparse lead--lag word projection.}
  \label{fig:lering}
\end{figure}
\FloatBarrier

\newpage
\bibliography{references}

\newpage
\appendix
\section{Representation of Words on a Computer}
Central to our computational approach is the ability to index signature terms by words and to perform operations on those words. Although the user supplies words as a nested iterable like
\begin{verbatim}
[(0, 3, 1), (5, 3, 2, 4), ...],
\end{verbatim}
we represent words internally as base-$d$ integers over a 0-based alphabet $\{0, \ldots, d-1\}$. For the truncated signature, this representation does not require any additional storage.

\begin{definition}[Word integer encoding]
For each $n\ge 1$, define the word-to-integer map
$$
\begin{aligned}
\phi_n:\W_n &\longrightarrow \{0,\ldots,d^n-1\},\\
w=(i_1,\ldots,i_n) &\longmapsto \sum_{j=1}^n i_j\, d^{\,n-j}.
\end{aligned}
$$
\end{definition}

For each fixed level $n$, the map $\phi_n$ is bijective by Proposition~\ref{prop:word_bijection}, so every word in $\W_n$ has a unique integer encoding. Adding the cumulative level offset then ensures that indices from different levels do not overlap, yielding a single consistent indexing of all words in $\W_{\leq N}$. 

\begin{proposition}
\label{prop:word_bijection}
The map $\phi_n: \W_n \to \{0, \ldots, d^n - 1\}$ is a bijection preserving lexicographic order: if $w_1 <_{\text{lex}} w_2$, then $\phi_n(w_1) < \phi_n(w_2)$.
\end{proposition}
\begin{proof}
This follows from the uniqueness of the base-$d$ representation.
\end{proof}

\subsection{Operations on Words}
Representing words as base-$d$ integers also allows for operations on words through simple arithmetic operations.

\begin{proposition}[Concatenation]\label{prop:concat_formula}
Let $u = (i_1, \ldots, i_k) \in \W_k$ and $v = (j_1, \ldots, j_m) \in \W_m$ be words with encoding $\phi_k(u)$ and $\phi_m(v)$. Then the encoding of the concatenated word $u \circ v = (i_1, \ldots, i_k, j_1, \ldots, j_m) \in \W_{k+m}$ is
$$
\phi_{k+m}(u \circ v) = \phi_k(u) \cdot d^m + \phi_m(v).
$$
\end{proposition}

\begin{proof}
Let $u = (i_1, \ldots, i_k)$ and $v = (j_1, \ldots, j_m)$. Then,
\begin{align*}
\phi_{k+m}(u \circ v) &= \sum_{r=1}^{k} i_r d^{k+m-r} + \sum_{s=1}^{m} j_s d^{m-s} \\
&= d^m \sum_{r=1}^{k} i_r d^{k-r} + \sum_{s=1}^{m} j_s d^{m-s} \\
&= d^m \cdot \phi_k(u) + \phi_m(v).
\end{align*}
\end{proof}

We can also easily get the encoding of prefix and suffix words, which, with its level offset, provides its unique index in the signature.

\begin{corollary}[Prefix extraction]\label{prop:prefix_extract}
Let $w = u \circ v \in \W_{k+m}$ where $u \in \W_k$ and $v \in \W_m$. Then,
$$
\phi_k(u) = \left\lfloor \frac{\phi_{k+m}(w)}{d^m} \right\rfloor.
$$
\end{corollary}

\begin{proof}
By Proposition~\ref{prop:concat_formula},
$$
\phi_{k+m}(w) = \phi_k(u) \cdot d^m + \phi_m(v).
$$
Since $0 \leq \phi_m(v) < d^m$, it follows that
$$
\phi_k(u) = \left\lfloor \frac{\phi_{k+m}(w)}{d^m} \right\rfloor.
$$
\end{proof}

\begin{corollary}[Suffix extraction]\label{prop:suffix_extract}
Let $w = u \circ v \in \W_{k+m}$ where $u \in \W_k$ and $v \in \W_m$. Then,
$$
\phi_m(v) = \phi_{k+m}(w) \bmod d^m.
$$
\end{corollary}

\begin{proof}
By Proposition~\ref{prop:concat_formula},
$$
\phi_{k+m}(w) = \phi_k(u) \cdot d^m + \phi_m(v).
$$
Since $0 \leq \phi_m(v) < d^m$, it follows that
$$
\phi_m(v) = \phi_{k+m}(w) \bmod d^m.
$$
\end{proof}

\subsection{Packing Letters}
In applying Chen's relation, each thread requires frequent access to the letters $\left(i_1, \ldots, i_n\right)$ of the word it is responsible for. To avoid expensive integer divisions and modulo operations in tight loops, we decode the letters once from the integer encoding and pack them into a single 64-bit integer:
$$
\text{packed\_letters}
=\sum_{j=1}^n i_j\,2^{b(j-1)},
\qquad \text{where}\qquad
\left\lceil\log _2 d\right\rceil \leq b \leq\lfloor 64 / n\rfloor.
$$

This can be reused for all prefixes and suffixes of the same word, and enables letter extraction using cheap bit shifts and masks.

\end{document}